\documentclass[10pt]{amsart}

\usepackage{amsfonts}
\usepackage{amsmath}
\usepackage{amssymb}
\usepackage{amsthm}
\usepackage{amscd}

\usepackage[all]{xy}

\begin{document}
\author{H\`a Quang Minh}

\title{Further properties of Gaussian  Reproducing Kernel Hilbert Spaces
}

\email{minh.haquang@iit.it}
\address{Istituto Italiano di Tecnologia (IIT), Via Morego 30, Genova 16163, Italy}
\begin{abstract} 
We generalize the orthonormal basis for the Gaussian RKHS described in \cite{MinhGaussian2010} to an infinite, continuously parametrized, family
of orthonormal bases, along with some implications. The proofs are direct generalizations of those in \cite{MinhGaussian2010}.
\end{abstract}


\maketitle

\newcommand{\inner}[2]{\ensuremath{\langle{#1},{#2}\rangle}}
\newcommand{\lmp}[2]{\ensuremath{\ell_{#2}^{#1}}}
\newcommand{\mapto}{\ensuremath{\rightarrow}}
\newcommand{\approach}{\ensuremath{\rightarrow}}
\newcommand{\imply}{\ensuremath{\Rightarrow}}
\newcommand{\inject}{\ensuremath{\hookrightarrow}}
\newcommand{\equivalent}{\ensuremath{\Longleftrightarrow}}
\newcommand{\inclusion}{\ensuremath{\hookrightarrow}}

\newtheorem{algorithm}{Algorithm}
\newtheorem{definition}{Definition}
\newtheorem{proposition}{Proposition}
\newtheorem{lemma}{Lemma}
\newtheorem{corollary}{Corollary}
\newtheorem{example}{Example}
\newtheorem{remark}{Remark}
\newtheorem{theorem}{Theorem}
\newtheorem{observation}{Observation}
\newtheorem{hypothesis}{Hypothesis}
\newtheorem{notation}{Notation}

\newcommand{\A}{\mathcal{A}}
\newcommand{\X}{\mathcal{X}}
\newcommand{\Y}{\mathcal{Y}}
\newcommand{\Z}{\mathbb{Z}}
\newcommand{\N}{\mathbb{N}}
\newcommand{\R}{\mathbb{R}}
\newcommand{\C}{\mathbb{C}}
\newcommand{\M}{\mathcal{M}}
\renewcommand{\P}{\mathcal{P}}
\renewcommand{\S}{\mathcal{S}}

\newcommand{\la}{\langle}
\newcommand{\ra}{\rangle}
\newcommand{\F}{\mathcal{F}}
\renewcommand{\H}{\mathcal{H}}
\renewcommand{\L}{\mathcal{L}}

\newcommand{\Fre}{Fr\'echet \;}
\newcommand{\Ga}{G\^ateaux \;}

\newcommand{\x}{\mathbf{x}}
\newcommand{\y}{\mathbf{y}}
\newcommand{\z}{\mathbf{z}}
\newcommand{\w}{\mathbf{w}}

\newcommand{\sgn}{\operatorname{{\mathrm sgn}}}

\def\span{{\rm span}}
\def\supp{{\rm supp}}

\markboth{}{} 
\pagestyle{myheadings}

\section{Main Results}\label{section:results}

\begin{notation}
Let $\alpha = (\alpha_1, \ldots , \alpha_n)\in (\N\cup \{0\})^n$, $|\alpha| =\sum_{j=1}^n\alpha_j$, $x^{\alpha} = x_1^{\alpha_1}\ldots x_n^{\alpha_n}$, and $C^d_{\alpha} = \frac{d!}{\alpha_1! \ldots \alpha_n!}$, the multinomial coefficients. Also, by writing $L^p(X)$, $dx$, we assume that the Lebesgue measure is being used.
\end{notation}

\begin{theorem}\label{theorem:HKGaussian} Let $X \subset \R^n$ be any set with non-empty interior. Let $K(x,t) = \exp({-\frac{||x-t||^2}{\sigma^2}})$. Let $c \in \R^n$ be fixed but otherwise arbitrary. Let $\H_K$ be the RKHS induced by $K$. Then $\dim(\H_K) = \infty$  and
\begin{equation}\label{equation:HKGaussian}
\H_K = \{f = e^{-\frac{||x-c||^2}{\sigma^2}}\sum_{|\alpha|=0}^{\infty}w_{\alpha}(x-c)^{\alpha} : ||f||^2_K = \sum_{k=0}^{\infty}\frac{k!}{(2/\sigma^2)^k} \sum_{|\alpha|=k}\frac{w_{\alpha}^2}{C^k_{\alpha}} < \infty\}.
\end{equation}
The inner product $\la \;,\;\ra_K$ on $\H_K$ is given by
\begin{displaymath}\label{equation:KproductGaussian}
\la f,g\ra_K = \sum_{k=0}^{\infty}\frac{k!}{(2/\sigma^2)^k} \sum_{|\alpha|=k}\frac{w_{\alpha}v_{\alpha}}{C^k_{\alpha}}
\end{displaymath}
for $f = e^{-\frac{||x-c||^2}{\sigma^2}}\sum_{|\alpha|=0}^{\infty}w_{\alpha}(x-c)^{\alpha}, g = e^{-\frac{||x-c||^2}{\sigma^2}}\sum_{|\alpha|=0}^{\infty}v_{\alpha}(x-c)^{\alpha} \in \H_K$. An orthonormal basis for $\H_K$ is
\begin{equation}\label{equation:orthonormalbasisGaussian}
\{\phi_{\alpha,c}(x) = \sqrt{\frac{(2/\sigma^2)^kC^k_{\alpha}}{k!}}e^{-\frac{||x-c||^2}{\sigma^2}}(x-c)^{\alpha}\}_{|\alpha|=k, k=0}^{\infty}.
\end{equation}
\end{theorem}

\begin{remark} 
By varying $c$ through $\R^n$, we obtain a continuously parametrized family of orthonormal bases of $\H_K$ - there are uncountably many of them. In particular, for $c=0$, we obtain the orthonormal basis
$\{\phi_{\alpha}(x) = \sqrt{\frac{(2/\sigma^2)^kC^k_{\alpha}}{k!}}e^{-\frac{||x||^2}{\sigma^2}}x^{\alpha}\}_{|\alpha|=k, k=0}^{\infty}$, already described in \cite{SteinwartGaussian2006} and \cite{MinhGaussian2010}.
\end{remark}
Let us discuss some immediate implications of  Theorem \ref{theorem:HKGaussian}. Consider the function
\begin{displaymath}
\phi_{0,c}(x) = \exp\left(-\frac{||x-c||^2}{\sigma^2}\right).
\end{displaymath} 
\begin{corollary}\label{corollary:HKbasis1}
For any $c \in \R^n$ and any set $X \subset \R^n$ with non-empty interior, $\phi_{0,c} \in \H_{K, \sigma}(X)$, with
\begin{equation}
||\phi_{0,c}||_{\H_{K, \sigma(X)}}  = 1.
\end{equation}
\end{corollary}
To illustrate the result of Corollary \ref{corollary:HKbasis1}, we note that 
by Aronszajn's Restriction Theorem (see \cite{Aronszajn}, section 5), for any set $X \subset \R^n$,
\begin{equation}
\H_{K, \sigma}(X) = \{f: X \mapto \R \; | \; \exists F \in \H_{K, \sigma}(\R^n): f = F|_X\},
\end{equation}
with corresponding norm
\begin{equation}
||f||_{H_{K, \sigma}(X)} = \min \{||F||_{\H_{K, \sigma}(\R^n)}: f = F|_X\}.
\end{equation}
In particular, this implies that $\phi_{0,c}(x) = \exp\left(-\frac{||x-c||^2}{\sigma^2}\right) \in \H_{K, \sigma}(X)$ for all $c \in \R^n$, with norm
\begin{equation}
||\phi_{0,c}||_{H_{K,\sigma}(X)} \leq ||\phi_{0,c}||_{\H_{K,\sigma}(\R^n)} = ||K_c||_{\H_{K, \sigma}(\R^n)} = 1.
\end{equation}
In particular, for $c \in X$, we have
\begin{equation}
||\phi_{0,c}||_{H_{K,\sigma}(X)} = ||K_c||_{\H_{K,\sigma}(X)} = 1.
\end{equation}
\begin{remark}
We wish to emphasize that one can only write $\phi_{0,c}(x) = \exp\left(-\frac{||x-c||^2}{\sigma^2}\right)  = K_c(x)$
when $c \in X$: the function $\exp\left(-\frac{||x-c||^2}{\sigma^2}\right)$ is always defined on any  $X \subset \R^n$ and
for any $c \in \R^n$, but we cannot talk about $K_c$ if $c \notin X$. Thus the Restriction Theorem only allows us
to conclude that $||\phi_{0,c}||_{H_{K,\sigma}(X)} \leq 1$ when $c \notin X$.
\end{remark}

The power of the Orthonormal Basis Theorem (Theorem \ref{theorem:RKHSbasis}) is clearly illustrated
in the proof of Theorem \ref{theorem:HKGaussian}. Note that there is no need for us to consider the larger set $\R^n$ or embedding maps between 
$H_{K,\sigma}(X)$ and $H_{K, \sigma}(\R^n)$. We automatically have $\phi_{\alpha,c} \in \H_{K, \sigma}(X)$ without having to invoke
the Restriction Theorem.

\begin{theorem}\label{theorem:theoremExp} Let $X \subset \R^n$ be any set with non-empty interior.
Let $K(x,z) = \exp({-\frac{||x-z||^2}{\sigma^2}})$. Let $c \in \R^n$ be arbitrary. The Hilbert space $\H_K$ induced by $K$ on $X$  contains the function $\exp(-\frac{\mu||x-c||^2}{\sigma^2})$ if and only if $0 < \mu < 2$. For such $\mu$, the corresponding functions have norms given by
\begin{displaymath}\label{equation:functionnorm1}
\left\|\exp\left(-\frac{\mu||x-c||^2}{\sigma^2}\right)\right\|_{\H_{K, \sigma}(X)}^2 = \left[\frac{1}{\mu (2-\mu)}\right]^{\frac{n}{2}}.
\end{displaymath}
\end{theorem}

In \cite{MinhGaussian2010}, it it shown that $f_0(x) = \exp(-\frac{\mu||x||^2}{\sigma^2}) \in \H_{K, \sigma}(X)$ 
if and only if $0 < \mu < 2$. If $X  = \R^n$, then it follows immediately from the translation-invariant property
that $f_c(x) = \exp(-\frac{\mu||x-c||^2}{\sigma^2}) \in \H_{K, \sigma}(X)$ if and only if $0 < \mu < 2$. Specifically, in terms of the Fourier Transform,
\begin{eqnarray}
||f_0||^2_{\H_{K, \sigma}(\R^n)} = \frac{1}{(2\pi)^{n}(\sigma \sqrt{\pi})^n}\int_{\mathbb{R}^n}e^{\frac{\sigma^2 |\xi|^2}{4}}{|\widehat{f_0}(\xi)|^2}d\xi\\
= \frac{1}{(2\pi)^{n}(\sigma \sqrt{\pi})^n}\int_{\mathbb{R}^n}e^{\frac{\sigma^2 |\xi|^2}{4}}{|\widehat{f_c}(\xi)|^2}d\xi 
= ||f_c||^2_{\H_{K, \sigma}(\R^n)}.
\end{eqnarray}
If $X \neq \R^n$, the Restriction Theorem gives
\begin{displaymath}
||f_c||^2_{\H_{K, \sigma}(X)} \leq ||f_c||^2_{\H_{K, \sigma}(\R^n)}.
\end{displaymath}
So if $0 < \mu < 2$, then we can conclude that $f_c \in \H_{K, \sigma}(X)$. But we cannot say more about $||f_c||^2_{\H_{K, \sigma}(X)}$. We cannot make a statement on the reverse direction either: if $f_c \in \H_{K, \sigma}(X)$, we cannot 
infer that $0 < \mu < 2$ using the Restriction Theorem. This is what Theorem \ref{theorem:theoremExp} gives us.

\section{Proofs of Main Results}\label{section:proofs}

\subsection{The Weyl Inner Product and Orthonormal Basis of the Gaussian RKHS} Let us prove Theorem \ref{theorem:HKGaussian}.
It was shown in \cite{CuckerSmale} that for $X=\R^n$, $n \in \N$, and $K(x,t) = \la x,t\ra^d$, $d \in \N$, we have $\H_K = \H_d(\R^n)$, the linear space of all homogeneous polynomials of degree $d$ in $\R^n$, with the inner product $\la \;,\;\ra$ being the Weyl inner product on $\H_d(\R^n)$:
\begin{displaymath}
\la f,g\ra_K = \sum_{|\alpha| =d}\frac{w_{\alpha}v_{\alpha}}{C^d_{\alpha}}.
\end{displaymath}
for $f = \sum_{|\alpha|=d}w_{\alpha}t^{\alpha}$, $g = \sum_{|\alpha| =d}v_{\alpha}t^{\alpha} \in \H_K$. 

\begin{theorem}[Aronszajn]\label{theorem:RKHSbasis} Let $H$ be a separable Hilbert space of functions over $X$ with orthonormal basis $\{\phi_k\}_{k=0}^{\infty}$. $H$ is a reproducing kernel Hilbert space iff
\begin{displaymath}
\sum_{k=0}^{\infty} |\phi_k(x)|^2 < \infty
\end{displaymath}
for all $x \in X$. The unique kernel $K$ is defined by
\begin{displaymath}
K(x,y) = \sum_{k=0}^{\infty} {\phi_k(x)} {\phi_k(y)}.
\end{displaymath}
\end{theorem}

\begin{proof}[\textbf{Proof of Theorem \ref{theorem:HKGaussian}}] We will show that the inner product $\la ,\ra_K$ in $\H_K$ is simply a generalization of the Weyl inner product for the homogeneous polynomial space $\H_d(\R^n)$, $d \in \N$. Consider the following expansion
\begin{eqnarray*}\nonumber\label{equation:expansionGaussian}
K(x,t) = \exp\left(-\frac{||x-t||^2}{\sigma^2}\right) = \exp\left(-\frac{||(x-c) - (t-c)||^2}{\sigma^2}\right) \\ =\exp\left(-\frac{||x-c||^2}{\sigma^2}\right)\exp\left(-\frac{||t-c||^2}{\sigma^2}\right)\sum_{k=0}^{\infty}\frac{(2/\sigma^2)^k}{k!}\sum_{|\alpha| =k}C^k_{\alpha}(x-c)^{\alpha}(t-c)^{\alpha}.
\end{eqnarray*}
Let 
$H_0 = \{f = e^{-\frac{||x-c||^2}{\sigma^2}}\sum_{|\alpha|=0}^{\infty}w_{\alpha}(x-c)^{\alpha} \; | \;\sum_{k=0}^{\infty}\frac{k!}{(2/\sigma^2)^k} \sum_{|\alpha|=k}\frac{w_{\alpha}^2}{C^k_{\alpha}} < \infty\}$. 
For $f \in H_0$, $g = e^{-\frac{||x-c||^2}{\sigma^2}}\sum_{|\alpha|=0}^{\infty}v_{\alpha}(x-c)^{\alpha} \in H_0$, we define the inner product
\begin{center}
$\la f,g\ra_{K,0} = \sum_{k=0}^{\infty}\frac{k!}{(2/\sigma^2)^k} \sum_{|\alpha|=k}\frac{w_{\alpha}v_{\alpha}}{C^k_{\alpha}}$.
\end{center}
Let us show that $H_0$ is itself a Hilbert space under $\la \;, \;\ra_{K,0}$. For simplicity let $n=1$. Then
\begin{displaymath}
H_0 =  \{f = e^{-\frac{(x-c)^2}{\sigma^2}}\sum_{k=0}^{\infty}w_k(x-c)^k \; | \;  \;\sum_{k=0}^{\infty}\frac{k!}{(2/\sigma^2)^k} w_k^2 < \infty\}.
\end{displaymath}
It is clear that $H_0$ is an inner product space under $\la \;,\; \ra_{K,0}$. Its completeness under the induced norm $||\;||_{K,0}$ is equivalent to the completeness of the weighted $\ell^2$ sequence space
\begin{displaymath}
\ell^2_{\sigma} = \{(w_k)_{k=0}^{\infty}: ||(w_k)_{k=0}^{\infty}||_{\ell^2_{\sigma}} = (\sum_{k=0}^{\infty}\frac{k!}{(2/\sigma^2)^k} w_k^2 )^{1/2}\},
\end{displaymath}
which is itself a Hilbert space. Thus $(H_0, ||\;||_{K,0})$ is a Hilbert space.

If $X \subset \R^n$ has non-empty interior, then the mononomials $(x-c)^{\alpha}$, $|\alpha| \geq 0$, are all distinct. It follows from the definition of the inner product $\la \;, \;\ra_{K,0}$  that the $\phi_{\alpha,c}$'s, as given in (\ref{equation:orthonormalbasisGaussian}), are orthonormal under $\la \;,\;\ra_{K,0}$. Since $H_0 = \span\{\phi_{\alpha,c}\}_{\alpha}$, it follows that the $\phi_{\alpha,c}$'s form an orthonormal basis for $(H_0, ||\;||_{K,0})$. By Theorem \ref{theorem:RKHSbasis} and the relations
\begin{center}
$\sum_{k=0}^{\infty}\sum_{|\alpha|=k}\phi_{\alpha,c}(x)\phi_{\alpha,c}(t) = K(x,t)$,
\end{center}
\begin{center}
$\sum_{k=0}^{\infty}\sum_{|\alpha|=k}|\phi_{\alpha,c}(x)|^2 = K(x,x) = 1 < \infty$,
\end{center}
it follows that $(H_0, ||\;||_{K,0})$ is a reproducing kernel Hilbert space of functions on $X$ with kernel $K(x,t)$. Since the RKHS induced by a kernel $K$ on a set $X$ is unique, we must have $(H_0, ||\;||_{K,0}) = (\H_K, ||\;||_K)$.
\end{proof}

\begin{proof}[\textbf{Proof of Theorem \ref{theorem:theoremExp}}]
Let us first consider the case $n=1$. Then
\begin{displaymath}
\H_K = \{f = e^{-\frac{(x-c)^2}{\sigma^2}}\sum_{k=0}^{\infty}w_k(x-c)^{k} :
||f||^2_K = \sum_{k=0}^{\infty}\frac{\sigma^{2k}k!}{2^k} {w_{k}^2} <
\infty\}
\end{displaymath}
Consider the function $e^{-\frac{\mu (x-c)^2}{\sigma^2}}$, which is
\begin{displaymath}
e^{-\frac{\mu (x-c)^2}{\sigma^2}} = e^{-\frac{(x-c)^2}{\sigma^2}}e^{-\frac{(\mu-1)(x-c)^2}{\sigma^2}} = e^{-\frac{(x-c)^2}{\sigma^2}} \sum_{k=0}^{\infty} (-1)^k\frac{(\mu-1)^k(x-c)^{2k}}{\sigma^{2k}k!}.
\end{displaymath}
Thus $w_{2k} = \frac{(-1)^k(\mu-1)^k}{\sigma^{2k}k!}$ and $w_j = 0$ for $j \neq
2k$. Then
\begin{eqnarray*}
\sum_{k=0}^{\infty}\frac{\sigma^{2k}k!}{2^k} {w_{k}^2} =
\sum_{k=0}^{\infty} \frac{\sigma^{4k} (2k)!}{2^{2k}}
\frac{(\mu-1)^{2k}}{\sigma^{4k} (k!)^2} =\sum_{k=0}^{\infty}
\frac{(\mu-1)^{2k}(2k)!}{2^{2k}(k!)^2}.
\end{eqnarray*}
If $\mu \leq 0$ or $\mu \geq 2$, then
\begin{displaymath}
\sum_{k=0}^{\infty}
\frac{(\mu-1)^{2k}(2k)!}{2^{2k}(k!)^2} \geq \sum_{k=0}^{\infty}
\frac{(2k)!}{2^{2k}(k!)^2} = \infty
\end{displaymath}
showing that $f \notin \H_K$ in those cases. If $0 < \mu < 2$, then
\begin{displaymath}
\sum_{k=0}^{\infty}
\frac{(\mu-1)^{2k}(2k)!}{2^{2k}(k!)^2} = 1+ \sum_{k=1}^{\infty}(\mu-1)^{2k}\frac{(2k-1)!!}{(2k)!!} 
\end{displaymath}
which converges by the Ratio Test. Hence we have
$\sum_{k=0}^{\infty}\frac{\sigma^{2k}k!}{2^k} {w_{k}^2} < \infty$,
showing that $e^{-\frac{\mu (x-c)^2}{\sigma^2}} \in \H_K$ for $0 < \mu < 2$, with norm 
\begin{displaymath}
 ||e^{-\frac{\mu (x-c)^2}{\sigma^2}}||_K^2 = \sum_{k=0}^{\infty}\frac{(\mu-1)^{2k}(2k)!}{2^{2k}(k!)^2} = \frac{1}{\sqrt{1-(\mu-1)^2}} = \frac{1}{\sqrt{\mu (2-\mu)}}.
\end{displaymath}
For any $n \in \N$, we have
\begin{eqnarray*}
e^{-\frac{\mu ||x-c||^2}{\sigma^2}} = e^{-\frac{||x-c||^2}{\sigma^2}}e^{-\frac{(\mu-1) ||x-c||^2}{\sigma^2}}= e^{-\frac{||x-c||^2}{\sigma^2}}\prod_{i=1}^n\sum_{k_i=0}^{\infty}w_{k_i}(x_i-c_i)^{k_i}\\
=e^{-\frac{||x-c||^2}{\sigma^2}}\sum_{\{k_1, \ldots, k_n\} = 0}^{\infty}\prod_{i=1}^nw_{k_i}(x_i-c_i)^{k_i}, 
\end{eqnarray*}
giving us
\begin{displaymath}
||e^{-\frac{\mu ||x-c||^2}{\sigma^2}}||^2_K = \sum_{\{k_1, \ldots, k_n\}=0}^{\infty}\prod_{i=1}^n\frac{\sigma^{2k_i}k_i!}{2^{k_i}}w_{k_i}^2 = \prod_{i=1}^n\sum_{k_i=0}^{\infty}\frac{\sigma^{2k_i}k_i!}{2^{k_i}}w_{k_i}^2.
\end{displaymath}
The result then follows from the one dimensional case above by symmetry.
\end{proof}

\bibliographystyle{plain}
\bibliography{cite_RKHS}

\end{document}